\documentclass{article}
\usepackage{ijcai23}

\usepackage{times}
\usepackage{soul}
\usepackage{url}
\usepackage[hidelinks]{hyperref}
\usepackage[utf8]{inputenc}
\usepackage[small]{caption}
\usepackage{amsmath}
\usepackage{amsthm}
\usepackage{graphicx}
\usepackage{booktabs}
\usepackage{algorithm}
\usepackage{amssymb}
\usepackage{color}
\usepackage{bbm}
\usepackage{comment}
\usepackage{algpseudocode}
\usepackage{cleveref}
\usepackage{multirow}
\usepackage[switch]{lineno}
\usepackage{subfigure}

\theoremstyle{definition}
\newtheorem{assumption}{Assumption}

\newtheorem{theorem}{Theorem}
\newtheorem{definition}{Definition}

\newtheorem{lemma}{Lemma}
\newtheorem{remark}{Remark}

\newtheorem{proposition}{Proposition}

\newcommand{\yulian}[1] {{\footnotesize\color{red}[Yulian: #1]}}
\newcommand{\R}{\mathbb{R}}
\renewcommand{\>}{\rangle}

\newcommand{\1}{\mathbbm{1}}
\renewcommand{\O}{\mathcal{O}}
\DeclareMathOperator{\E}{\mathbb{E}}

\title{Quantum Heavy-tailed Bandits}

\author{
 Yulian Wu$^1$\thanks{ The first two authors contributed equally.}
\and
Chaowen Guan$^2$\footnotemark[1]
\and
Vaneet Aggarwal$^3$\And
Di Wang$^1$
\affiliations
$^1$KAUST\\
$^2$The Pennsylvania State University\\
$^3$Purdue University 
\emails
yulian.wu@kaust.edu.sa,
cmg6558@psu.edu,
vaneet@purdue.edu,
di.wang@kaust.edu.sa
}

\begin{document}

\maketitle

\begin{abstract}
    In this paper, we study multi-armed bandits (MAB) and stochastic linear bandits (SLB) with heavy-tailed rewards and quantum reward oracle. Unlike the previous work on quantum bandits that assumes bounded/sub-Gaussian distributions for rewards, here we investigate the quantum bandits problem under a weaker assumption that the distributions of rewards only have bounded $(1+v)$-th moment for some $v\in (0,1]$. In order to achieve regret improvements for heavy-tailed bandits, we first propose a new quantum mean estimator for heavy-tailed distributions, which is based on the Quantum Monte Carlo Mean Estimator  and achieves a quadratic improvement of estimation error compared to the classical one. Based on our quantum mean estimator, we focus on quantum heavy-tailed MAB and SLB and propose quantum algorithms based on the Upper Confidence Bound (UCB) framework for both problems with $\Tilde{O}(T^{\frac{1-v}{1+v}})$ regrets, polynomially improving the dependence in terms of $T$ as compared to classical (near) optimal regrets of $\Tilde{O}(T^{\frac{1}{1+v}})$, where $T$ is the number of rounds. Finally, experiments also support our theoretical results and show the effectiveness of our proposed methods.
\end{abstract}

\section{Introduction}

As a fundamental model to deal with the uncertainty in decision-making problems, bandits problem, originally introduced by Thompson in 1933 \cite{thompson1933likelihood}, has wide applications including recommendation system \cite{tang2013automatic},  dynamic pricing \cite{cohen2020feature}, and medicine \cite{gutierrez2017multi}, to name a few. A bandit problem is a sequential game between a learner and an environment. The game is played over $T$ rounds. In each round $t \in [T]$, the learner first chooses an action from a given set, and the environment then reveals a reward. The objective for the learner is to choose actions that lead to the largest possible cumulative reward over $T$ rounds, i.e., to minimize the cumulative regret which is defined as the difference between the maximum expected reward and the expected reward collected by the learner.

While there are tremendous  studies on bandits problem, most of them only consider the classical setting. Bandits problem in the quantum setting has received much attention recently due to its superiority to the classical setting on the bound of regret. For example, \cite{casale2020quantum} and \cite{wang2021quantum} provide the first study on the exploration of quantum multi-armed bandits (MAB) with binary rewards. \cite{lumbreras2022multi} studies the trade-off between exploration and exploitation in quantum MAB  under bounded rewards assumption, where a  the lower bound of  regret as $\Omega(\sqrt{T})$ was shown. Recently, \cite{wan2022quantum} studies the problems of MAB and stochastic linear bandits (SLB) with quantum reward oracle under the assumption that the reward distributions are bounded or have bounded variance. Specifically, it shows that it is possible to achieve a logarithmic regret, which is an exponential improvement as compared to the optimal rate of $\tilde{\Theta}(\sqrt{T})$ (if we omit other terms) in the classical setting.

Although there are some results on quantum bandits problem, most of them rely on the assumptions that the reward distributions are light-tailed, such as bounded or with bounded variance. However, in a wide variety of real-world online decision-making systems such as financial portfolio \cite{bradley2003financial}, online user behavior \cite{kumar2010characterization} and  localization error \cite{ruotsalainen2018error}, rewards are generated from heavy-tailed distributions. 
Although there is no existing bandits algorithm that is  designed for handling heavy-tailed rewards in the quantum setting, several work has studied the problem in 
the classical setting. \cite{bubeck2013bandits} first investigates the problem of classical stochastic MAB with heavy-tailed rewards where the reward distribution of each action has finite $(1+v)$-th moment with $v\in (0, 1]$. Under this assumption, several extensions  have been studied including linear bandits \cite{medina2016no,shao2018almost,xue2020nearly,xue2020nearly}, pure exploration (best arm identification) \cite{yu2018pure}, Lipschitz bandits \cite{lu2019optimal}, private MAB \cite{tao2022optimal}, 
and Bayesian optimization \cite{ray2019bayesian}. Under the finite $(1+v)$-th raw moment assumption, \cite{lee2020optimal} shows that the optimal rate of MAB is $\Tilde{\Theta}(T^{\frac{1}{1+v}})$ with respect to $T$. For SLB with infinite arms, under the heavy-tailed setting, \cite{shao2018almost} establishes an upper bound of $\Tilde{O}(dT^{\frac{1}{1+v}})$ and provides an $\Omega(d T^{\frac{1}{1+v}})$ lower bound, where $d$ is the dimension of contextual information.

 Based on the above discussions on quantum bandits with light-tailed rewards  and classical heavy-tailed bandits, a natural question is:
 
\noindent \textit{What are the theoretical behaviors of heavy-tailed bandits in the quantum setting, and what are the improvements of regret as compared to the classical setting?}

In order to answer these questions, in this paper, we focus on the algorithms and theoretical regret bounds for heavy-tailed MAB and SLB with quantum reward oracle where the reward distribution of each action  has bounded $(1+v)$-th moment for some $v\in (0,1]$. To the best of our knowledge, we are the first to study quantum heavy-tailed bandits. Specifically, our contributions can be summarized as follows (see Table \ref{tab:1} for details):
\begin{itemize}
    \item To design algorithms for quantum heavy-tailed bandits, we first focus on  quantum mean estimation for one-dimensional heavy-tailed distributions. Specifically, based on the Quantum Monte Carlo Mean Estimator \cite{montanaro2015quantum}, we develop the Quantum Truncated Mean Estimator (QTME). QTME could achieve an estimation error of $\Tilde{O}\left({n^{-\frac{2v}{1+v}}}\right)$, which  quadratically improves  the classical error bound of $\Tilde{O}\left({n^{-\frac{v}{1+v}}}\right)$ in \cite{bubeck2013bandits}, where $n$ is the number of samples (or the quantum oracle complexity in the quantum setting). To the best of our knowledge, this is the first result on quantum mean estimation for heavy-tailed distributions, and it can be applied in other related problems. 
    
    \item Based on QTME, we design the first quantum version of Upper Confidence Bound (UCB)-type algorithms for MAB with heavy-tailed rewards. Specifically, we propose a quantum batch UCB algorithm. We also prove the regret bound of $\tilde{O}\left(T^{\frac{1-v}{1+v}}\right)$ for the algorithm,  which improves a  factor of  $\tilde{O}(T^{\frac{v}{1+v}})$ as compared to the classical one.
    \item  Then we develop the first quantum  algorithm and provide theoretical bounds for SLB with heavy-tailed rewards. To obtain improvement for heavy-tailed SLB, we design a UCB-type algorithm, namely Heavy-QLinUCB, based on QTME and the weighted least square estimator. We prove a regret bound which achieves an improvement of factor $\tilde{O}(T^{\frac{v}{1+v}})$ on $T$ as compared to the classical one.
    \item Finally, experiments also support our theoretical results and outperform the classical algorithms.
\end{itemize}
Due to the space limit, all proofs and additional experiments are included in Appendix of Supplementary Materials.

\begin{table*}[htb]
    \centering
\begin{tabular}{ccccc}
\hline\textbf{ Model} & \textbf{Reference} & \textbf{Setting} & \textbf{Assumption} & \textbf{Regret} \\
\hline \hline \multirow{5}{*}{MAB} & \cite{lattimore2020bandit} & Classical & sub-Gaussian & $\Theta(\sqrt{K T})$ \\
\cline{2-5}&\cite{bubeck2013bandits,lee2020optimal} &Classical& heavy-tail& $\Theta((K \log T)^{\frac{v}{1+v}}T^{\frac{1}{1+v}})$\\
\cline{2-5}  & \cite{wan2022quantum}& Quantum & bounded value & $O\left(K \log (T)\right)$ \\
\cline{2-5} & \cite{wan2022quantum} & Quantum & bounded variance & $O\left(K \log ^{5 / 2}(T) \log \log (T)\right)$ \\
\cline{2-5} & {\bf This paper} (Theorem \ref{thm:reg1})&Quantum  & heavy-tail & $O\left(K T^{\frac{1-v}{1+v}} \log T \right) $\\
\hline \hline \multirow{5}{*}{SLB}  & \cite{lattimore2020bandit} & Classical & sub-Gaussian & $\widetilde{\Theta}(d \sqrt{T})$ \\
\cline{2-5} &\cite{shao2018almost}& Classical & heavy-tail& $\Tilde{\Theta}(T^{\frac{1}{1+v}})$\\ 
\cline{2-5}  & \cite{wan2022quantum} & Quantum & bounded value & $O\left(d^2 \log ^{5 / 2}(T)\right)$ \\
\cline{2-5}  & \cite{wan2022quantum} & Quantum & bounded variance & $O\left(d^2 \log ^4(T) \log \log (T)\right)$ \\
\cline{2-5} & {\bf This paper} (Theorem \ref{thm:LinearReg})& Quantum& heavy-tail & $O\left( d^2   T^{\frac{1-v}{1+v}} (\log T)^\frac{3}{2} \left( \log\log T \right)^{\frac{2v}{1+v}} 
    \right)$\\
\hline \hline
\end{tabular}
\caption{Regret bounds on multi-armed bandits (MAB) and stochastic linear bandits (SLB).}
\label{tab:1}
\end{table*}

\section{Related Work}
Table \ref{tab:1} shows key comparisons to the previous work on classical heavy-tailed MAB and SLB, and the work on quantum MAB and SLB (with light-tailed reward distributions). 

\noindent \textbf{Quantum Mean Estimation.} There is a series of quantum mean estimators for bounded   random variables \cite{grover1998framework,brassard2011optimal,brassard2002quantum,abrams1999fast}. For distributions with bounded variance, \cite{hamoudi2021quantum} develops a mean estimator which achieves a quadratic improvement on the estimation error as compared to the classical one. \cite{montanaro2015quantum}  also presents the Quantum Monte Carlo method to estimate the mean of random variables with bounded variance. Such method is later applied by \cite{wan2022quantum} in quantum bandits with bounded rewards to get confidence upper bounds  of rewards, which is the key point of its UCB-type algorithms. However, all of those methods cannot handle heavy-tailed distributions. 

\noindent \textbf{Quantum Bandits.} \cite{casale2020quantum} and \cite{wang2021quantum} focus on quantum versions of the pure exploration (best-arm identification) problem in bandits with binary rewards. \cite{lumbreras2022multi} first studies the trade-off between exploration and exploitation in MAB with properties of quantum states under the bounded rewards assumption. \cite{wan2022quantum} studies MAB and SLB with quantum reward oracle and proposes quantum algorithms for both problems with logarithmic regrets under the bounded reward or variance assumption. \cite{wang2021quantum2} studies the  quantum improvement of the Markov decision process problem in reinforcement learning. However, all  these results  assume that the reward distributions are either bounded or sub-Gaussian.

\noindent \textbf{Classical Heavy-tailed Bandits.}  \cite{bubeck2013bandits} provides the first study on stochastic MAB with heavy-tailed rewards (in the classical setting) under the assumption that the rewards distributions have finite $(1+v)$-th moments for $v\in (0,1]$, and proposes a UCB-type algorithm. \cite{medina2016no} extends the analysis to SLB, and develops two algorithms with $\widetilde{O}\left(d T^{\frac{2+v}{2(1+v)}}\right)$ and $\widetilde{O}\left(\sqrt{d} T^{\frac{1+2 v}{1+3 v}}+d T^{\frac{1+v}{1+3 v}}\right)$ regret bounds respectively. In a subsequent work, \cite{shao2018almost} presents a lower bound of $\Omega(dT^{\frac{1}{1+v}})$ for SLB with heavy-tailed
rewards assuming that the arm set is infinite. It also develops algorithms  with  near optimal regret upper bounds of $O(dT^{\frac{1}{1+v}})$. \cite{xue2020nearly} establishes an upper bound of $\Tilde{O}(d^{\frac{1}{2}}T^{\frac{1}{1+v}})$ for two novel algorithms and provides an $\Omega(d^{\frac{v}{1+v}}T^{\frac{1}{1+v}})$ lower bound for heavy-tailed SLB with finite arms.

\section{Preliminaries}

\subsection{Quantum Computation}
\textbf{Notations and Basics.}
A quantum state can be seen as a vector $\vec{x} = (x_1, x_2, \dots, x_m)^\top$ in Hilbert space $\mathbb{C}^m$ such that $\sum_i |x_i|^2 = 1$. We follow the Dirac bra/ket notation on quantum states, i.e.,  we denote the quantum state for $\vec{x}$ by $|x\rangle$ and denote $\vec{x}^{\dagger}$ by $\langle x|$ , where $\dagger$ means the Hermitian conjugation. 

Given a state $|x\> = \sum^m_{i=1} x_i |i\>$, we call $x_i$ the amplitude of the state $|i\>$. Given two quantum states $|x\rangle \in \mathbb{C}^m$ and $|y\rangle \in \mathbb{C}^m$, we denote their tensor product by $\langle x|y \rangle:= \sum_i x^\dagger_i y_i$. Given $|x\> \in \mathbb{C}^m$ and $|y\> \in \mathbb{C}^n$, we denote their tensor product by $|x\>|y\> := (x_1y_1, \cdots, x_m y_n)^\top$.

A quantum algorithm works by applying a sequence of unitary operators to an input quantum state. In many cases, the construction of input states would require information from a unitary operator which is called a \emph{quantum oracle}. This operator can be accessed multiple times by a quantum algorithm. Hence, the \emph{quantum query complexity} of a quantum algorithm is defined as the number of a quantum oracle used by the quantum algorithm.

\noindent \textbf{Quantum Mean Estimation.} Now we review an existing quantum mean estimator for one-dimensional bounded random variables. 
In fact, there are multiple well-known estimators \cite{montanaro2015quantum,hamoudi2021quantum,terhal1999quantum}, where most of them are  applications and adaptations of the amplitude estimation algorithm \cite{brassard2002quantum}. In this paper,  we  design and present a quantum mean estimator for heavy-tailed distributions  using the approach in \cite{montanaro2015quantum} which is presented  in Theorem \ref{thm:qme}. We note that alternative approaches could be used together with our truncation ideas for estimation, e.g., the Bernoulli estimator presented in \cite{hamoudi2021quantum}, to have a variant of QME. First we introduce the following quantum oracle. 
Consider a random variable $Y$ with its finite sample space $\Omega$,  we consider the quantum oracle $\mathcal{O}_Y$
with 
\begin{equation}\label{eq:oracle}
    \mathcal{O}_Y : |0\> \rightarrow \sum_{y\in \Omega} \sqrt{\Pr[Y=y]} |y\>|\psi_y\>, 
\end{equation}
where  $|\psi_y\>$ is some normalized state.

\begin{theorem}[Quantum Monte Carlo Mean Estimator \cite{montanaro2015quantum}]\label{thm:qme}
Assume that $Y : \Omega \rightarrow \R$ is a random variable in the interval $[0,1]$, $\Omega$ is equipped with a probability measure $P$, and a quantum oracle $\O_Y$ encoding $P$ and $Y$ that has the form of (\ref{eq:oracle}). There is a quantum algorithm $QME(t, \delta, \O_Y)$ that queries $\O_Y$ and $\O_Y^\dag$ at most $O(t)$ times and outputs an estimate $\hat{Y}$ such that with probability at least $1-\delta$
\[
|\hat{Y} - \E[Y]| \leq  C \left( \frac{\sqrt{\E[Y]} \log(1/\delta)}{t} + \frac{\log^2(1/ \delta)}{t^2} \right), 
\] 
where $C>0$ is a universal constant. 
\end{theorem}

\subsection{Bandits with Heavy-tailed Rewards}
\label{subsec:PreBandits}
In a multi-armed bandit (MAB) problem, a learner is faced repeatedly with a choice among $K$ different actions over $T$ rounds. After each choice $a_t \in [K]$ at round $t \in [T]$, the learner receives a numerical reward $x_t$ 
 which is i.i.d. sampled from a stationary but unknown probability distribution $P_{a_t}$ that depends on the selected action. Denote by $\mu_a$ the mean of each distribution $P_{a}$ for $a\in [K]$, and by  $\mu^*=\max_{a\in [K]}\mu_a$ the maximum among all expectations $\{\mu_a\}_{a\in[k]}$. The objective of the learner in classical MAB is to maximize the expected total reward over $T$ rounds, i.e., to minimize the expected cumulative regret which is defined as
\begin{equation}
    \mathcal{R}_T\triangleq T\mu^*-\mathbb{E}\left[\sum\limits_{t=1}^{T}{x_t}\right]=\sum_{t=1}^T(\mu^*-\mu_{a_t}),
\end{equation}
where the expectation is taken with respect to all the randomness of the algorithm. In this paper, we consider a heavy-tailed setting where each arm's reward distribution has bounded $(1+v)$-th raw moment for some $v \in (0,1]$. Concretely, we assume that there is a constant $u>0$ such that for each reward distribution $P_a$,
\begin{equation}\label{heavyDefi}
    \mathbb{E}_{X\sim P_a} [|X|^{1+v}] \le u.
\end{equation}
In this paper, we assume both $v$ and $u$ are known constants. We note that both of the raw moment and central moment assumptions have been studied in previous work on classical heavy-tailed  multi-armed bandit problem \cite{bubeck2013bandits}. Here, we claim that, finite raw moment implies that the central moment is finite, and vice versa. See  the proof of Lemma 10 in \cite{tao2022optimal} for details.

In a stochastic linear bandit (SLB) problem, a learner can play actions from a fixed action set $\mathcal{A} \subseteq \mathbb{R}^d$. There is an unknown parameter $\theta^* \in \mathbb{R}^d$ which determines the mean reward of each action. At round $t$, the leaner is given a decision set $\mathcal{A}_t \subseteq \mathbb{R}^d$, from which she/he chooses an action $a_t \in \mathcal{A}_t$ and receives reward $x_t \in \mathbb{R}$. The expected reward of action $a$ is $\mathbb{E}[x_t]=\langle \theta^*, a_t \rangle $. We also assume that there is a constant $u>0$ such that for each reward distribution $P_a$ with $a\in \mathcal{A}$,
\begin{equation}\label{heavyDefi2}
    \mathbb{E}_{X\sim P_a} [|X|^{1+v}] \le u. 
\end{equation}
Similar to the  MAB case,  here we also assume both $v$ and $u$ are known constants. It is often assumed that each action $a$ and the $\theta^*$ has bounded $\ell_2$-norm, i.e., there are some parameters $L, S > 0$ such that
\begin{equation}
\|a\|_2 \leq L \text { for all } a \in \mathcal{A} \text {, and }\left\|\theta^*\right\|_2 \leq S \text {. }
\end{equation}

Let $a^*=\operatorname{argmax}_{a \in \mathcal{A}} a^{\top} \theta^*$ be the action with the largest expected reward. The same as MAB, SLB also has T rounds.  The goal is again to minimize the cumulative regret
\begin{equation}
\mathcal{R}(T)=\sum_{t=1}^T\left(a^*-a_t\right)^{\top} \theta^* .
\end{equation}

\noindent In the quantum version of bandits problems, the intermediate sample rewards are replaced by a chance to access an unitary oracle $\mathcal{O}_a$ or its inverse which encodes  the reward distribution $P_a$ of the selected arm $a$. 
Following the previous work on quantum bandits \cite{wan2022quantum}, we consider the following quantum reward oracle, which is a special case of \eqref{eq:oracle}. 

\noindent \textbf{Quantum Reward Oracle.} This oracle is a  generalization of MAB and SLB to the quantum world. The \emph{Quantum Multi-armed Bandits} (QMAB) and \emph{Quantum Stochastic Linear Bandits} (QSLB) problems are defined basically following the framework of classical bandits problems, but with the quantum reward oracle described below. Let $P_i$ denotes the reward distribution of the selected arm $i$ and $\Omega_i$ denotes a \emph{finite} sample space of the distribution $P_i$. Formally, the reward oracle is defined as follows:
\begin{equation}\label{eq:banditoracle}
    \mathcal{O}_i : |0\> \rightarrow \sum_{\omega \in \Omega_i} \sqrt{P_i(\omega)} |\omega\>|y_i(\omega)\>,
\end{equation}
where $y_i : \Omega_i \rightarrow \R$ is the random reward associated with arm $i$. $\O_i$  encodes the probability $P_i$ and the random variable $y_i$. At any round $t$, the algorithm chooses an arm $i_t$ by invoking either of the unitary oracle $\O_{i_t}$ or $\O^\dag_{i_t}$ at most once.

\section{Quantum Mean Estimator for Heavy-tailed Distributions}\label{sec:mean}
In this section, we present our quantum mean estimator for (one-dimensional) heavy-tailed distributions. First, we develop a truncation-based  estimator (\Cref{{alg:basic}}) that estimates the mean of positive-valued heavy-tailed random variables. Then with this algorithm as a subroutine, we present our quantum mean estimator without the positiveness assumption. 

To approximate the mean of a positive-valued and heavy-tailed random variable $X$ with quantum oracle (\ref{eq:oracle}), QBME (\Cref{alg:basic}) estimates the mean of the part of $X$ that lies in the interval $[0,B]$, where $B$ will  be set according to the upper bound of the $(1+v)$-th raw moment of $X$ (\Cref{heavyDefi}).  QBME divides the interval $[0,B]$ into multiple segments and invokes QME (from \Cref{thm:qme}) to estimate the means of these disjoint segments respectively. Finally, it returns a linear combination of these means. We note that the idea of dividing the interval into several segments also has been adapted from  Section 2.2 of \cite{montanaro2015quantum}, which is a generalization of a result in \cite{heinrich2002quantum}. However, here we extend the results to the heavy-tailed distribution case. 

The main weakness of QBME is it can only handle non-negative random variables. However, in general, heavy-tailed random variables can have both negative and positive values. Hence, to estimate the mean of such a general random variable $X$, we propose QTME (\Cref{alg:mean_est}) that estimates the means of $X \1_{X \geq 0}$ and $X \1_{X <0}$ respectively where each can be computed using QBME, and outputs the sum of them. In the following we provide estimation errors for these two algorithms. 

\begin{algorithm}
\caption{Quantum Basic Mean Estimator ($\mathsf{QBME}(\mathcal{O}, n,c, B, \delta)$)}\label{alg:basic}
\begin{algorithmic}[1]
\Require Quantum oracle $\mathcal{O}$ in (\ref{eq:oracle}), an upper bound $B$, constant $c$ and input $n$ determine there are 
$c n \log^{3/2}(1/\delta)$ oracle queries to  $\mathcal{O}$ of $X$ satisfying \eqref{heavyDefi}, and failure probability $\delta>0$ 
\State Let $k \leftarrow \log(1/\delta)$ and $t \leftarrow c n \sqrt{\log(1/\delta)}$ where $c$ is a constant. 
\State Let $a_{-1} \leftarrow 0$ and $a_\ell \leftarrow \frac{2^\ell}{n} B$.
\For{ $0 \leq \ell \leq k$}
\State Estimate the scaled mean $\frac{\E[X \cdot \1_{a_{\ell-1} \leq X < a_\ell}]}{a_\ell}$ via $\mathsf{QME}(t, \delta, \mathcal{O}_{\frac{X}{a_\ell}\cdot \1_{a_{\ell-1} \leq X < a_\ell}})$ in \Cref{thm:qme}, denote the estimator as $\widehat{\mu}_\ell$.  
\EndFor
\State Output $\widehat{\mu} \leftarrow \sum^k_{\ell =0} a_\ell \cdot \widehat{\mu}_\ell$.
\end{algorithmic}
\end{algorithm}


\begin{theorem}\label{thm:basic}
Let $X$ be a random variable satisfying \eqref{heavyDefi}. 
Algorithm \ref{alg:basic}, $\mathsf{QBME}(\mathcal{O}, n,c, B, \delta)$, has quantum query complexity of $\tilde{O}(n)$ and its output $\hat{\mu}$
satisfies that  
\begin{align*}
|\hat{\mu} - \mu| 
&\leq \frac{\sqrt{2} \cdot \sqrt{u B^{1-v}} \cdot \log(1/ \delta) }{c n} + \frac{3B \log^2(1/ \delta)}{c n^2 \sqrt{\log(1/\delta)}} + \frac{u}{B^v} \\
\end{align*}
with probability at least $1-\delta$, where $c$ is a universal constant. In particular, by setting $B=\left(\frac{\sqrt{u}n}{\log (1/\delta)}\right)^{\frac{2}{1+v}}$, we have
\[
|\hat{\mu} - \mu|    \leq O\left( \frac{u^\frac{1}{1+v}(\log\frac{1}{\delta})^{\frac{2v}{1+v}}}{n^{\frac{2v}{1+v}}} \right). 
\]
\end{theorem}

\begin{algorithm}
\caption{Quantum Truncated Mean Estimator ($\mathsf{QTME}(\mathcal{O}, n,c,B, \delta)$)}\label{alg:mean_est}
\begin{algorithmic}[1]
\Require  Quantum oracle $\mathcal{O}$ in (\ref{eq:oracle}), an upper bound $B$, constant $c$ and input $n$ determine there are 
$c n \log^{3/2}(1/\delta)$ oracle queries to  $\mathcal{O}$ of $X$ satisfying \eqref{heavyDefi}, and failure probability $\delta>0$ 
\State Define the non-negative random variables $Y_+ \leftarrow X \1_{X \geq 0}$ and $Y_- \leftarrow - X \1_{X < 0}$.
\State Compute an estimate $\widehat{\mu}_{Y_+}$ of $\E[Y_+]$ with $\mathsf{QBME}(\mathcal{O},n,c, B, \delta)$ and an estimate $\hat{\mu}_{Y_-}$ of $\E[Y_-]$ with $\mathsf{QBME}(\mathcal{O}, n,c, B, \delta)$ respectively.
\State Output $\widehat{\mu} \leftarrow \widehat{\mu}_{Y_+} + \widehat{\mu}_{Y_-}$.
\end{algorithmic}
\end{algorithm}


\begin{theorem}[Mean estimation]\label{thm:mean_est}
Let $X$ be a random variable satisfying \eqref{heavyDefi}. 
The truncated mean estimator $\mathsf{QTME}(\mathcal{O}, n ,c, B, \delta)$ with {$B=\left(\frac{\sqrt{u}n}{\log (1/\delta)}\right)^{\frac{2}{1+v}}$} has quantum query complexity $\tilde{O}(n)$ and  outputs a mean estimator ${\hat{\mu}}$ such that with probability at least $1-\delta$,
\begin{equation}\label{eq:MeanEst}
    |{\hat{\mu}}-\mu| \le O\left( \frac{u^\frac{1}{1+v}(\log\frac{1}{\delta})^{\frac{2v}{1+v}}}{n^{\frac{2v}{1+v}}} \right).
\end{equation}
\end{theorem}

\begin{remark}
    We note that the oracle query complexity in Algorithm \ref{alg:basic} is $O(n\log^\frac{3}{2}(1/\delta))$, which is used for our analysis in later sections. We can also fix the number of oracle queries  $m$ as the input and get almost the same upper bounds as in Theorem \ref{thm:basic} (replacing $n$ by $m$) up to some logarithmic factors. This means that if  each oracle query corresponds to one sample in the classical setting, then with some parameter $B$, the output of Algorithm \ref{alg:mean_est} could achieve a bound of $\tilde{O}(m^{-\frac{2v}{1+v}})$ for a given number of oracle queries, $m$.  As compared to the estimation error bound of $O(m^{-\frac{v}{1+v}})$ for heavy-tailed distributions in the classical setting given by \cite{bubeck2013bandits} with $m$ samples, we can see our result in above theorem achieves a quadratic improvement on $m$ (up to some logarithmic factors), which is the key point for achieving regret improvement  in the later sections. Moreover, when $v=1$, we can recover the result of the mean estimation error of $\tilde{O}\left(\frac{\log(1/\delta)}{m}\right)$ for distributions with bounded variance in \cite{montanaro2015quantum}.
\end{remark}

\section{Quantum Multi-armed Bandits with Heavy-tailed rewards}\label{sec:MAB}
In this section, we present an algorithm for QMAB with heavy-tailed rewards and show its regret  bound. 
The framework of  our algorithm is Upper Confidence Bound (UCB), 
which is a canonical method to balance exploitation and exploration in bandit learning. In classical MAB with bounded rewards \cite{auer2002finite,agrawal1995sample}, at round $t$ the UCB framework chooses an arm by 
\begin{equation}\label{eq:classUCB}
    a_t=\arg\max_{a} \hat{\mu}_a+C\sqrt{\frac{\log (1/\delta)}{N_t(a)}},
\end{equation}
where $N_t(a)$ is the pull number of arm $a$ until round $t$ and the empirical mean $\hat{\mu}_a$ represents the exploitation term. This means the action that currently has the highest estimated reward will be the chosen action. The second term of above equation  gives exploration. That means an action has not been tried very often is more likely to be selected. The hyperparameter $C$ controls the level of exploration. 
We note that \eqref{eq:classUCB} comes from the fact that the true mean of arm $a$, $\mu_a$, satisfies $\mu_a \in \left[\hat{\mu}_a-C\sqrt{\frac{\log (1/\delta)}{N_t(a)}},\hat{\mu}_a+C\sqrt{\frac{\log (1/\delta)}{N_t(a)}}\right]$ with probability at least $1-\delta$,  which is by the Hoeffding Lemma. The  confidence interval length  $O\left(\frac{1}{\sqrt{N_t(a)}}\right)$ makes the UCB algorithm obtain a  regret of $O(\sqrt{T})$. \cite{wan2022quantum} improves the length of confidence interval to $O\left(\frac{1}{N_t(a)}\right)$ by the Quantum Monte Carlo method in \cite{montanaro2015quantum} for quantum bandits with bounded rewards so that they can achieve logarithmic regret.

In classical bandits with heavy-tailed rewards \cite{bubeck2013bandits}, the key point of designing robust UCB algorithms for the problem is to replace the empirical mean in \eqref{eq:classUCB} by robust 
 mean estimators. Since  heavy-tailed random variables are unbounded and  even could have infinite variance, \cite{bubeck2013bandits} adapts a truncation-based method for heavy-tailed distributions with  finite raw moment and get a confidence radius of $O\left(\frac{1}{N_t(a)^{\frac{v}{1+v}}}\right)$ by Bernstein's inequality \cite{vershynin2018high} in Lemma \ref{alemma3}. Motivated by Theorem \ref{thm:mean_est}, 
 we  can quadratically improve the length of confidence interval  to $O\left(\frac{1}{N_t(a)^{\frac{2v}{1+v}}}\right)$ for heavy-tailed bandits with quantum reward oracle, which leads to a regret improvement of $O(T^{\frac{1-v}{1+v}})$.

\begin{algorithm}[htb]
\caption{Heavy-QUCB}
\begin{algorithmic}[1]
 \Require Quantum reward oracles $\mathcal{O}_i$ for all $i \in [K]$ in \eqref{eq:banditoracle}, a constant $C$, number of rounds $T$ and fail probability $\delta$, $v$ and $u$ are parameters of reward distributions in (\ref{heavyDefi})
 \For {$i=1,\dots,K$}
     \State  $N_i \leftarrow 1$ and $\beta_i \leftarrow \frac{u^\frac{1}{1+v}(\log\frac{1}{\delta})^{\frac{2v}{1+v}}}{C \cdot N_i^{\frac{2v}{1+v}}}$.
     \State Play arm $i$ for the next $CN_i\log^{3/2}(1/\delta)$ rounds.
     \State Run $\mathsf{QTME}(\mathcal{O}_i,N_i,C, B_{N_i},\delta)$.
 \EndFor
 \For {Each epoch $s=1,2,\dots$ (terminate when we have used $T$ queries to all $\mathcal{O}_i$) }
     \State Let $i_s \leftarrow \operatorname{argmax}_i \hat{\mu}(i)+\beta_i$.
     \State Update $N_{i_s} \leftarrow 2N_{i_s}$ and $\beta_{i_s} \leftarrow \frac{u^\frac{1}{1+v}(\log\frac{1}{\delta})^{\frac{2v}{1+v}}}{C \cdot N_{i_s}^{\frac{2v}{1+v}}}$.
     \State Play $i_s$ for the next $CN_{i_s}\log^{3/2}(1/\delta)$ rounds, update $\hat{\mu}(i_s)$ by running $\mathsf{QTME}(\mathcal{O}_{i_s},N_{i_s},C, B_{N_{i_s}},\delta)$. 
 \EndFor
\end{algorithmic}
\label{Algo:Heavy-QMAB}
\end{algorithm}

Unlike the UCB framework for classical bandits where we can use observed rewards to calculate the  empirical mean estimator $\hat{\mu}$, there is another challenge for quantum bandits learning. That is, before we make a measurement on the quantum state, we cannot observe any reward. That means our quantum algorithm cannot observe rewards in each round. Here we use the ``doubling trick" to overcome the challenge. Specifically, we first divide the total number of rounds in to several epochs. 
In each epoch, we double the pull number of each arm from the last epoch. Then for each arm we only need to invoke QTME with its pull number (instead of rewards) as input to get a quantum version of its reward mean estimation.

The ``doubling trick" has been also used in phased UCB for classical bandits learning, such as in \cite{azize2022privacy}. 
The reason here we use ``doubling trick" is that in phased UCB, the empirical means are computed only using rewards of one phase. To yield a near-optimal regret, the number of rounds in a phase should be greater than previous phases so that the confidence bound will be tighter as phase grows. By using the ``doubling trick", we only need to input the necessary number of quantum oracle queries to QTME in each epoch, then we can get the quantum mean estimation of rewards without observing each reward.

Combining with all the above ideas, we propose the Heavy-QUCB algorithm for quantum MAB with heavy-tailed rewards based on QTME, phased UCB, and the ``doubling trick". The details of the algorithm are provided in Algorithm \ref{Algo:Heavy-QMAB}. The key idea of the algorithm is that we adaptively divide whole $T$ rounds into several phases. During each phase, we choose an arm $i_s$ by the UCB strategy in Line 7. Then, we double $N_{i_s}$ so that the confidence radius is reduced by $2^{-\frac{2v}{1+v}}$. The learner plays $i_s$ for the next $C N_{i_s} \log^{\frac{3}{2}}(1/\delta)$ rounds and invokes QTME (Algorithm \ref{alg:mean_est}) to update the quantum mean estimator  $\hat{\mu}(i_s)$. After this, the algorithm  will go to the next phase.  The algorithm will terminate after $T$ rounds.

\begin{theorem}[Regret bound]
\label{thm:reg1}
    Under the heavy-tailed assumption of \eqref{heavyDefi}, for $\delta=\frac{1}{T}$ the cumulative regret of Algorithm \ref{Algo:Heavy-QMAB} satisfies $$\mathcal{R}_T \le O\left(u^\frac{1}{1+v}  K T^{\frac{1-v}{1+v}} \log T \right).$$
\end{theorem}

 \begin{remark}
 As compared to the regret bound of $\widetilde{O}(T^{\frac{1}{1+v}})$ for classical heavy-tailed bandits in \cite{bubeck2013bandits},  our bound is lower  by a factor of $O(T^{\frac{v}{1+v}})$. Moreover, when $v=1$ we can achieve a  logarithmic regret of $O(K\log T)$, which is the same as in \cite{wan2022quantum}. However, it is notable that the assumption to achieve $O(K\log T)$ in \cite{wan2022quantum} is that the reward distributions are bounded,  while here we only need reward distributions have bounded second order raw moments. 
 Note that \cite{wan2022quantum} also provides a regret of $O\left(K \log ^{5 / 2}(T) \log \log (T)\right)$ for quantum MAB with  rewards that have bounded variance. However, such result is incomparable to ours even for when $v=1$ since the  variance  in \cite{wan2022quantum} is the second order \textit{central} moment while we consider the  $(1+v)$-th \textit{raw} moment for $v\in (0,1]$. To summarize, even  for $v=1$, our result improves and generalizes the previous ones.  
 \end{remark}
\vspace{-0.1in}
\section{Quantum Stochastic Linear Bandits with Heavy-tailed rewards}\label{sec:SLB}
In this section, we provide an algorithm named Heavy-QLinUCB for MSLB with heavy-tailed rewards. 
Similar to the classical SLB, we allow the action set be infinite.

In the quantum version of SLB with heavy-tailed rewards, we also encounter a similar problem as in quantum MAB. That is, before we get a measurement on the quantum state, we cannot observe any reward. In order to solve the problem, we will use the weighted least square when we estimate $\theta^*$ in SLB so that we can get a linear dependence of different arms. Then a careful choice of weights will help us to get a variant of the ``doubling trick" to solve the problem introduced by our quantum subroutine.

\begin{algorithm}[htb]
\caption{Heavy-QLinUCB}
\begin{algorithmic}[1]
 \Require  Quantum reward oracles $\mathcal{O}_i$ for all $i \in [K]$ in \eqref{eq:banditoracle}, a constant $C$, number of rounds $T$ and fail probability $\delta$, $v$ and $u$ are parameters of reward distributions in (\ref{heavyDefi})
 \State Initialize $V_0 \leftarrow \lambda I_d,$ $ \hat{\theta}_0 \leftarrow \mathbf{0} \in \mathbb{R}^d$ and $m \leftarrow d \log \left(\frac{L^2 T^{\frac{4v}{1+v}}}{d \lambda}+1\right)$.

 \For {each epoch $s=1,2,\dots$(terminate when we have used $T$ queries to all $\mathcal{O}_i$) }
     \State $\mathcal{C}_{s-1} \leftarrow\{\theta \in \mathbb{R}^d:\|\theta-\hat{\theta}_{s-1}\|_{V_{s-1}} \leq \lambda^{1 / 2} S+\sqrt{d (s-1)}\}$.
 \State $(a_s,\tilde{\theta}_s )\leftarrow \operatorname{argmax}_{ (a,\theta) \in  \mathcal{A} \times \mathcal{C}_{s-1}} a^{\top} \theta$.
 \State $\epsilon_s \leftarrow\left\|a_s\right\|_{V_{s-1}^{-1}}$.
 \State $N_s \leftarrow \frac{C u^{\frac{1}{2v}}\log(m/\delta)}{\epsilon_s^{\frac{1+v}{2v}}}$.
 \For{the next $C N_s \log^{3/2}(1/\delta)$ rounds }
      \State Play action $a_s$ and run $\mathsf{QTME}(\mathcal{O}_{s},N_{s},C, B_{N_{s}},\delta/m)$, getting $x_s$ as an estimation of $a_s^\top \theta^*$.
   \EndFor
 \State Denote $A_s \leftarrow\left(a_1, a_2, \ldots, a_s\right)^{\top} \in \mathbb{R}^{s \times d},$ $ X_s \leftarrow\left(x_1, x_2, \ldots, x_s\right)^{\top} \in \mathbb{R}^s$ and $W_s \leftarrow \operatorname{diag}\left(\frac{1}{\epsilon_1^2}, \frac{1}{\epsilon_2^2}, \ldots, \frac{1}{\epsilon_s^2}\right)$.
 \State Update $V_s \leftarrow V_{s-1}+\frac{1}{\epsilon_s^2} a_s a_s^{\top}$ and $\hat{\theta}_s \leftarrow V_s^{-1} A_s^{\top} W_s X_s$.
 \EndFor
\end{algorithmic}
\label{Algo:Heavy-QLinMAB}
\end{algorithm}

The details of Heavy-QLinUCB are showed in Algorithm \ref{Algo:Heavy-QLinMAB}, which adapts the Linear UCB framework. It runs in several epochs, and $m$ in Line 1 is an upper bound for the number of total epochs. 
In epoch $s$, the algorithm  first constructs a confidence set $\mathcal{C}_{s-1}$ for the underlying parameter $\theta^*$ in Line 3. Based on the confidence set, it picks the best arm $a_s$ among the arm set $\mathcal{A}$ in Line 4. After $a_s$ is chosen, it determines a carefully selected accuracy value $\epsilon_s$ for epoch $s$, and then the algorithm plays arm $a_s$ for the next $C N_s \log^{\frac{3}{2}}(1/\delta)$ rounds (Line 5-9).  When playing arm $a_s$ in epoch $s$, the algorithm implements the quantum algorithm $\mathsf{QTME}(\mathcal{O}_{s},N_{s},C, B_{N_{s}},\delta/m)$ to get a quantum estimator $x_s$ for rewards such that $|a_s^\top \theta^* - x_s| \le \epsilon_s$ where the failure probability is less than $\delta/m$. Next, in Line 10-11, it updates the estimate $\hat{\theta}_s$ of $\theta^*$ using a weighted least square estimator. That is
\begin{equation}
\label{WeiLSE}
\hat{\theta}_s=\underset{\theta \in \Theta}{\operatorname{argmin}} \sum_{k=1}^s \frac{1}{\epsilon_k^2}\left\|a_k^{\top} \theta-x_k\right\|_2^2+\lambda\|\theta\|_2^2,
\end{equation}
where $\lambda$ is a regularization parameter.

\begin{figure*}[!ht]
    \centering
    \subfigure{
      \includegraphics[width=0.32 \textwidth]{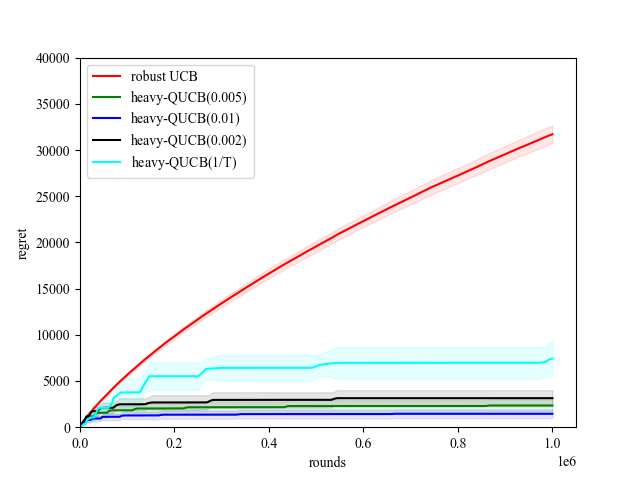}}
      \subfigure{\includegraphics[width=0.32\textwidth]{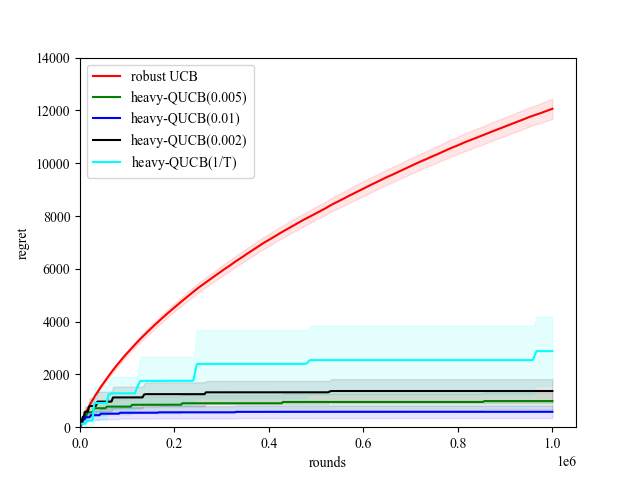}}
      \subfigure{\includegraphics[width=0.32\textwidth]{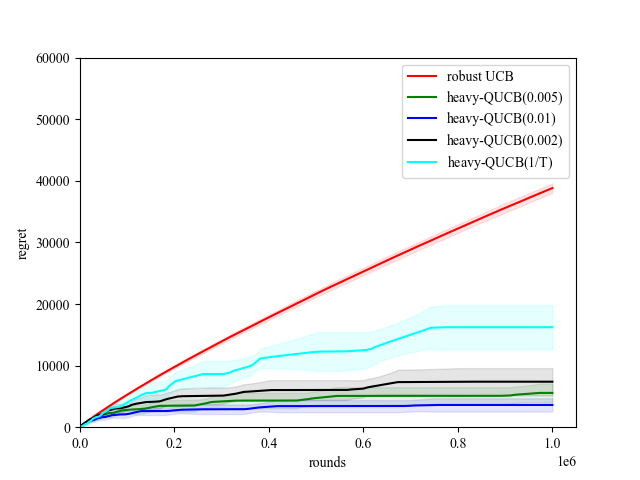}}
    \caption{Comparison between robust UCB and Heavy-QUCB with $v=0.5$, Heavy-QUCB ($1/T$) means $\delta = 1/T$}
    \label{fig:qmab_0.5}
\end{figure*}

The estimator \eqref{WeiLSE} has simple a closed-form solution as
follows. Let
$$V_s=\lambda I+\sum_{k=1}^s \frac{1}{\epsilon_k^2} a_k a_k^{\top}=\lambda I+A_s^{\top} W_s A_s \in \mathbb{R}^{d \times d}.$$
 Then,
$
\hat{\theta}_s:=V_s^{-1} A_s^{\top} W_s X_s,
$
where $A_s, X_s, W_s$ are defined in Line 10 in Algorithm \ref{Algo:Heavy-QLinMAB}. 
It is notable that 
we  need to carefully set $\epsilon_s=\left\|a_s\right\|_{V_{s-1}^{-1}}$ where $V_{s-1}$ is calculated in epoch $s-1$. This choice makes the determinant of $V_s$, $\text{det}(V_s)$, be double compared with the last epoch. Therefore,  Heavy-QLinUCB uses an implicit ``doubling trick" which makes the number of epochs be $O(d \log T)$ (Lemma \ref{lem:stage}), controls the estimation error of $x_s$ and $\epsilon_s$, and enforces 
 the confidence interval of  $\theta^*$  to decrease with the epoch increasing (Lemma \ref{lem:confiSet}). Specifically, we have the following two lemmas. 

\begin{lemma}
\label{lem:stage}
    Algorithm \ref{Algo:Heavy-QLinMAB} has at most $m= d \log \left(\frac{L^2 T^{\frac{4v}{1+v}}}{d \lambda}+1\right)$ epochs if $v\in [\frac{1}{3},1]$,  where $\lambda$ is the regularization parameter in \eqref{WeiLSE}.
\end{lemma}

\begin{lemma}
\label{lem:confiSet}
    With probability at least $1-\delta$, for all $s \geq 0, \theta^* \in \mathcal{C}_s:=\left\{\theta \in \mathbb{R}^d:\left\|\theta-\hat{\theta}_s\right\|_{V_s} \leq\right.$ $\left.\lambda^{1 / 2} S+\sqrt{d s}\right\}$.
\end{lemma} 

\begin{theorem}[Regret bound]
\label{thm:LinearReg}
    Under the  heavy-tailed assumption \ref{heavyDefi2} for rewards with $v\in [\frac{1}{3}, 1]$, with probability at least $1-\delta$, the upper bound of the expected regret of  Algorithm \ref{Algo:Heavy-QLinMAB}, $\mathbb{E}[\mathcal{R}_T]$, is
   $O( d^2  u^{\frac{1}{1+v}} \log^{\frac{3}{2}} (\frac{L^2 T^{\frac{4v}{1+v}}}{d \lambda}+1)  T^{\frac{1-v}{1+v}} \cdot ( \log\frac{d \log (\frac{L^2 T^{\frac{4v}{1+v}}}{d \lambda}+1)}{\delta})^{\frac{2v}{1+v}}).$
\end{theorem}

\begin{remark}
    Note that when $v=1$, we can recover the result in  \cite[Theorem 3]{wan2022quantum} for quantum SLB with bounded  rewards. On the other hand, compared with the optimal regret bound of $\Theta(dT^{\frac{1}{1+v}})$ for classical heavy-tailed SLB with infinite arms in \cite{shao2018almost}, our regret for quantum heavy-tailed SLB with infinite arms in above theorem improves a factor of $\tilde{O}(T^{\frac{v}{1+v}})$ for  term $T$.
\end{remark}

\vspace{-0.2in}
\section{Experiments}\label{sec:expe}
In this section, we conduct experiments to demonstrate the performance of our two quantum bandit algorithms. For all experiments, we repeat 100 times and calculate the average regret and the standard deviation. Our experiments are executed on a computer equipped with AMD Ryzen 7 7300X CPU and 16GB memory. Due to the space limit, here we only show partial results for QMAB. Additional results of QMAB and all details for  QSLB are shown in \Cref{apx:experiment}.

\noindent \textbf{QMAB settings.} For the synthetic reward generation, we
follow the same settings as in the previous work on (classical) MAB
with heavy-tailed rewards \cite{lee2020optimal,tao2022optimal}. Specifically, we set $K=5$ and consider three instances $S_1$, $S_2$, and $S_3$, whose reward means all are restricted in the interval $[0.1, 1]$: For $S_1$ we let the mean gaps of reward distributions of sub-optimal arms decrease
linearly where the largest mean is always 0.9 and the
smallest mean is always 0.1 (thus the means of rewards in $S_1$ are \{0.1, 0.3, 0.5, 0.7, 0.9\}); In $S_2$, we consider an instance that a larger
fraction of arms have large sub-optimal gaps, where
we set the mean of each arm a by a quadratic convex
function $\mu_a=0.05(a-5)^2+0.1$; In $S_3$, we set the mean of each arm $a$  by a 
concave function $\mu_a=-0.05(a-1)^2+0.9$. Compared to $S_2$, $S_3$ has a larger fraction of arms with small sub-optimal gaps. In each instance, the rewards of each arm $i \in [K]$ are sampled from a Pareto distribution with shape parameter $\alpha$ and scale parameter $\lambda_i$. We set $\alpha = 1.05 + v$ which guarantees that the
$1+v$-th moment of reward distribution always exists
and is upper bounded by $u = \frac{\alpha \lambda_i^{1+v}}{\alpha - (1+v)}$ in \Cref{heavyDefi}. For a given $\alpha$, the mean is $\frac{\alpha \lambda_i}{\alpha -1}$ and this implies $\lambda_i = \frac{(\alpha-1)\mu_i}{\alpha}$. We take the maximum of $\frac{\alpha \lambda_i^{1+v}}{\alpha - (1+v)}$ among all arms as $u$ in \Cref{heavyDefi}. We set $v$ to be either 0.2 or 0.5 and fix the number of rounds $T$ as $10^6$. 

\noindent {\bf Results.} To show the superiority of QMAB to the classical MAB, we compare our method with the robust UCB (with $\delta=\frac{1}{T}$) in \cite{bubeck2013bandits}. Besides that, for our method heavy-QUCB, we also consider different failure probabilities $\delta=\{0.005, 0.01, 0.002, \frac{1}{T}\}$. The results for $v = 0.5$ and $v = 0.2$ are shown in \Cref{fig:qmab_0.5}  and \Cref{fig:qmab_0.2} in \Cref{apx:experiment}, respectively.  From the figures, we can see that Heavy-QUCB has much lower expected regret than robust UCB,
and such discrepancy becomes more obvious when $v$ becomes larger. This is due to the fact that theoretically we show heavy-QUCB can improve a factor of $\tilde{O}(T^{\frac{v}{1+v}})$. Moreover, for each method we can see larger $v$ will make the expected regret smaller, this is due to that the regret bound is either $\tilde{O}(T^\frac{1}{1+v})$ or $\tilde{O}(T^\frac{1-v}{1+v})$. Moreover, smaller failure probability will make the regret become larger, which can also be observed in our regret analysis.  We can also see similar phenomenons in SLB. In summary, all experimental results corroborate our theories.  


\vspace{-0.1in}
\section{Conclusions}\vspace{-0.05in}
We investigated the problems of quantum multi-armed bandits (QMAB) and stochastic linear bandits (QSLB)
with heavy-tailed rewards. To get a confidence radius of mean estimation of rewards with the quantum reward oracle, we first proposed a novel quantum mean estimation method namely QTME for heavy-tailed random variables which achieves a quadratic improvement on  estimation error compared with the classical one. Based on our novel quantum mean estimator, we proposed a UCB-based algorithm named Heavy-QUCB  for heavy-tailed MAB with quantum reward oracle and established a regret bound of $\Tilde{O}(T^{\frac{1-v}{1+v}})$, where $T$ is the number of rounds. For QSLB, we developed the method of Heavy-QLinUCB based on the Linear UCB framework and show a similar regret bound. Finally, our theoretical results are supported by experimental results, which  show the superiority of our algorithms to the classical ones.


\bibliographystyle{named}
\bibliography{ijcai22}

\begin{thebibliography}{}

\bibitem[\protect\citeauthoryear{Abrams and Williams}{1999}]{abrams1999fast}
Daniel~S Abrams and Colin~P Williams.
\newblock Fast quantum algorithms for numerical integrals and stochastic
  processes.
\newblock {\em arXiv preprint quant-ph/9908083}, 1999.

\bibitem[\protect\citeauthoryear{Agrawal}{1995}]{agrawal1995sample}
Rajeev Agrawal.
\newblock Sample mean based index policies by o (log n) regret for the
  multi-armed bandit problem.
\newblock {\em Advances in Applied Probability}, 27(4):1054--1078, 1995.

\bibitem[\protect\citeauthoryear{Auer \bgroup \em et al.\egroup
  }{2002}]{auer2002finite}
Peter Auer, Nicolo Cesa-Bianchi, and Paul Fischer.
\newblock Finite-time analysis of the multiarmed bandit problem.
\newblock {\em Machine learning}, 47(2):235--256, 2002.

\bibitem[\protect\citeauthoryear{Azize and Basu}{2022}]{azize2022privacy}
Achraf Azize and Debabrota Basu.
\newblock When privacy meets partial information: A refined analysis of
  differentially private bandits.
\newblock {\em arXiv preprint arXiv:2209.02570}, 2022.

\bibitem[\protect\citeauthoryear{Bradley and
  Taqqu}{2003}]{bradley2003financial}
Brendan~O Bradley and Murad~S Taqqu.
\newblock Financial risk and heavy tails.
\newblock In {\em Handbook of heavy tailed distributions in finance}, pages
  35--103. Elsevier, 2003.

\bibitem[\protect\citeauthoryear{Brassard \bgroup \em et al.\egroup
  }{2002}]{brassard2002quantum}
Gilles Brassard, Peter Hoyer, Michele Mosca, and Alain Tapp.
\newblock Quantum amplitude amplification and estimation.
\newblock {\em Contemporary Mathematics}, 305:53--74, 2002.

\bibitem[\protect\citeauthoryear{Brassard \bgroup \em et al.\egroup
  }{2011}]{brassard2011optimal}
Gilles Brassard, Frederic Dupuis, Sebastien Gambs, and Alain Tapp.
\newblock An optimal quantum algorithm to approximate the mean and its
  application for approximating the median of a set of points over an arbitrary
  distance.
\newblock {\em arXiv preprint arXiv:1106.4267}, 2011.

\bibitem[\protect\citeauthoryear{Bubeck \bgroup \em et al.\egroup
  }{2013}]{bubeck2013bandits}
S{\'e}bastien Bubeck, Nicolo Cesa-Bianchi, and G{\'a}bor Lugosi.
\newblock Bandits with heavy tail.
\newblock {\em IEEE Transactions on Information Theory}, 59(11):7711--7717,
  2013.

\bibitem[\protect\citeauthoryear{Casal{\'e} \bgroup \em et al.\egroup
  }{2020}]{casale2020quantum}
Balthazar Casal{\'e}, Giuseppe Di~Molfetta, Hachem Kadri, and Liva Ralaivola.
\newblock Quantum bandits.
\newblock {\em Quantum Machine Intelligence}, 2(1):1--7, 2020.

\bibitem[\protect\citeauthoryear{Cohen \bgroup \em et al.\egroup
  }{2020}]{cohen2020feature}
Maxime~C Cohen, Ilan Lobel, and Renato Paes~Leme.
\newblock Feature-based dynamic pricing.
\newblock {\em Management Science}, 66(11):4921--4943, 2020.

\bibitem[\protect\citeauthoryear{Grover}{1998}]{grover1998framework}
Lov~K Grover.
\newblock A framework for fast quantum mechanical algorithms.
\newblock In {\em Proceedings of the thirtieth annual ACM symposium on Theory
  of computing}, pages 53--62, 1998.

\bibitem[\protect\citeauthoryear{Guti{\'e}rrez \bgroup \em et al.\egroup
  }{2017}]{gutierrez2017multi}
Benjam{\'\i}n Guti{\'e}rrez, Lo{\"\i}c Peter, Tassilo Klein, and Christian
  Wachinger.
\newblock A multi-armed bandit to smartly select a training set from big
  medical data.
\newblock In {\em International Conference on Medical Image Computing and
  Computer-Assisted Intervention}, pages 38--45. Springer, 2017.

\bibitem[\protect\citeauthoryear{Hamoudi}{2021}]{hamoudi2021quantum}
Yassine Hamoudi.
\newblock Quantum sub-gaussian mean estimator.
\newblock {\em arXiv preprint arXiv:2108.12172}, 2021.

\bibitem[\protect\citeauthoryear{Heinrich}{2002}]{heinrich2002quantum}
Stefan Heinrich.
\newblock Quantum summation with an application to integration.
\newblock {\em Journal of Complexity}, 18(1):1--50, 2002.

\bibitem[\protect\citeauthoryear{Kumar and
  Tomkins}{2010}]{kumar2010characterization}
Ravi Kumar and Andrew Tomkins.
\newblock A characterization of online browsing behavior.
\newblock In {\em Proceedings of the 19th international conference on World
  wide web}, pages 561--570, 2010.

\bibitem[\protect\citeauthoryear{Lattimore and
  Szepesv{\'a}ri}{2020}]{lattimore2020bandit}
Tor Lattimore and Csaba Szepesv{\'a}ri.
\newblock {\em Bandit algorithms}.
\newblock Cambridge University Press, 2020.

\bibitem[\protect\citeauthoryear{Lee \bgroup \em et al.\egroup
  }{2020}]{lee2020optimal}
Kyungjae Lee, Hongjun Yang, Sungbin Lim, and Songhwai Oh.
\newblock Optimal algorithms for stochastic multi-armed bandits with heavy
  tailed rewards.
\newblock {\em Advances in Neural Information Processing Systems},
  33:8452--8462, 2020.

\bibitem[\protect\citeauthoryear{Lu \bgroup \em et al.\egroup
  }{2019}]{lu2019optimal}
Shiyin Lu, Guanghui Wang, Yao Hu, and Lijun Zhang.
\newblock Optimal algorithms for lipschitz bandits with heavy-tailed rewards.
\newblock In {\em International Conference on Machine Learning}, pages
  4154--4163. PMLR, 2019.

\bibitem[\protect\citeauthoryear{Lumbreras \bgroup \em et al.\egroup
  }{2022}]{lumbreras2022multi}
Josep Lumbreras, Erkka Haapasalo, and Marco Tomamichel.
\newblock Multi-armed quantum bandits: Exploration versus exploitation when
  learning properties of quantum states.
\newblock {\em Quantum}, 6:749, 2022.

\bibitem[\protect\citeauthoryear{Medina and Yang}{2016}]{medina2016no}
Andres~Munoz Medina and Scott Yang.
\newblock No-regret algorithms for heavy-tailed linear bandits.
\newblock In {\em International Conference on Machine Learning}, pages
  1642--1650. PMLR, 2016.

\bibitem[\protect\citeauthoryear{Montanaro}{2015}]{montanaro2015quantum}
Ashley Montanaro.
\newblock Quantum speedup of monte carlo methods.
\newblock {\em Proceedings of the Royal Society A: Mathematical, Physical and
  Engineering Sciences}, 471(2181):20150301, 2015.

\bibitem[\protect\citeauthoryear{Ray~Chowdhury and
  Gopalan}{2019}]{ray2019bayesian}
Sayak Ray~Chowdhury and Aditya Gopalan.
\newblock Bayesian optimization under heavy-tailed payoffs.
\newblock {\em Advances in Neural Information Processing Systems}, 32, 2019.

\bibitem[\protect\citeauthoryear{Ruotsalainen \bgroup \em et al.\egroup
  }{2018}]{ruotsalainen2018error}
Laura Ruotsalainen, Martti Kirkko-Jaakkola, Jesperi Rantanen, and Maija
  M{\"a}kel{\"a}.
\newblock Error modelling for multi-sensor measurements in infrastructure-free
  indoor navigation.
\newblock {\em Sensors}, 18(2):590, 2018.

\bibitem[\protect\citeauthoryear{Shao \bgroup \em et al.\egroup
  }{2018}]{shao2018almost}
Han Shao, Xiaotian Yu, Irwin King, and Michael~R Lyu.
\newblock Almost optimal algorithms for linear stochastic bandits with
  heavy-tailed payoffs.
\newblock {\em Advances in Neural Information Processing Systems}, 31, 2018.

\bibitem[\protect\citeauthoryear{Tang \bgroup \em et al.\egroup
  }{2013}]{tang2013automatic}
Liang Tang, Romer Rosales, Ajit Singh, and Deepak Agarwal.
\newblock Automatic ad format selection via contextual bandits.
\newblock In {\em Proceedings of the 22nd ACM international conference on
  Information \& Knowledge Management}, pages 1587--1594, 2013.

\bibitem[\protect\citeauthoryear{Tao \bgroup \em et al.\egroup
  }{2022}]{tao2022optimal}
Youming Tao, Yulian Wu, Peng Zhao, and Di~Wang.
\newblock Optimal rates of (locally) differentially private heavy-tailed
  multi-armed bandits.
\newblock In {\em International Conference on Artificial Intelligence and
  Statistics}, pages 1546--1574. PMLR, 2022.

\bibitem[\protect\citeauthoryear{Terhal}{1999}]{terhal1999quantum}
Barbara Terhal.
\newblock {\em Quantum algorithms and quantum entanglement}.
\newblock PhD thesis, University of Amsterdam, 1999.

\bibitem[\protect\citeauthoryear{Thompson}{1933}]{thompson1933likelihood}
William~R Thompson.
\newblock On the likelihood that one unknown probability exceeds another in
  view of the evidence of two samples.
\newblock {\em Biometrika}, 25(3-4):285--294, 1933.

\bibitem[\protect\citeauthoryear{Vershynin}{2018}]{vershynin2018high}
Roman Vershynin.
\newblock {\em High-dimensional probability: An introduction with applications
  in data science}, volume~47.
\newblock Cambridge university press, 2018.

\bibitem[\protect\citeauthoryear{Wan \bgroup \em et al.\egroup
  }{2022}]{wan2022quantum}
Zongqi Wan, Zhijie Zhang, Tongyang Li, Jialin Zhang, and Xiaoming Sun.
\newblock Quantum multi-armed bandits and stochastic linear bandits enjoy
  logarithmic regrets.
\newblock {\em arXiv preprint arXiv:2205.14988}, 2022.

\bibitem[\protect\citeauthoryear{Wang \bgroup \em et al.\egroup
  }{2021a}]{wang2021quantum2}
Daochen Wang, Aarthi Sundaram, Robin Kothari, Ashish Kapoor, and Martin
  Roetteler.
\newblock Quantum algorithms for reinforcement learning with a generative
  model.
\newblock In {\em International Conference on Machine Learning}, pages
  10916--10926. PMLR, 2021.

\bibitem[\protect\citeauthoryear{Wang \bgroup \em et al.\egroup
  }{2021b}]{wang2021quantum}
Daochen Wang, Xuchen You, Tongyang Li, and Andrew~M Childs.
\newblock Quantum exploration algorithms for multi-armed bandits.
\newblock In {\em Proceedings of the AAAI Conference on Artificial
  Intelligence}, volume~35, pages 10102--10110, 2021.

\bibitem[\protect\citeauthoryear{Xue \bgroup \em et al.\egroup
  }{2020}]{xue2020nearly}
Bo~Xue, Guanghui Wang, Yimu Wang, and Lijun Zhang.
\newblock Nearly optimal regret for stochastic linear bandits with heavy-tailed
  payoffs.
\newblock {\em arXiv preprint arXiv:2004.13465}, 2020.

\bibitem[\protect\citeauthoryear{Yu \bgroup \em et al.\egroup
  }{2018}]{yu2018pure}
Xiaotian Yu, Han Shao, Michael~R Lyu, and Irwin King.
\newblock Pure exploration of multi-armed bandits with heavy-tailed payoffs.
\newblock In {\em UAI}, 2018.

\end{thebibliography}

\newpage
\onecolumn
\appendix

\section{Useful Lemmas}

\begin{lemma}[Bernstein's Inequality \cite{vershynin2018high}]\label{alemma3}
Let $X_1, \cdots X_n$ be $n$ independent zero-mean random variables. Suppose $|X_i|\leq M$ and $\mathbb{E}[X_i^2]\leq s$ for all $i$. Then for any $t>0$, we have 
\begin{equation*}
    \mathbb{P}\{\frac{1}{n}\sum_{i=1}^n X_i \geq t\}\leq \exp(-\frac{\frac{1}{2}t^2n}{s+\frac{1}{3}Mt})
\end{equation*}
\end{lemma}

\begin{lemma}
\label{lem:ineq1}
    Let $f(x)=(\sum_{k=1}^m (a_k)^x)^{\frac{1}{x}}$ where $a_k >0$ and $x\in (0,+\infty)$, then $f(x)$ is a decreasing function on $x$.
\end{lemma}
\begin{proof}[\bf Proof of Lemma \ref{lem:ineq1}]
    Let $g(x)=\ln f(x)= \frac{1}{x} \ln \left ( \sum_{k=1}^m (a_k)^x\right)$. Then 
    $$\begin{aligned}
         \frac{\mathrm{d} g(x) }{\mathrm{d} x} &= -\frac{1}{x^2} \ln \left ( \sum_{k=1}^m (a_k)^x\right)+\frac{1}{x} \frac{\sum_{k=1}^m(a_k)^x\ln a_k}{\sum_{k=1}^m (a_k)^x}\\
         & =\frac{\sum_{k=1}^m(a_k)^x\ln (a_k)^x - \sum_{k=1}^m(a_k)^x\ln \left(\sum_{k=1}^m(a_k)^x\right)}{x^2\sum_{k=1}^m (a_k)^x}\\
         &= \frac{\sum_{k=1}^m(a_k)^x\ln \left(\frac{(a_k)^x}{\sum_{k=1}^m(a_k)^x}\right)}{x^2\sum_{k=1}^m (a_k)^x}\\
         & < 0.
    \end{aligned}$$
    Thus, we get the result.
\end{proof}

\section{Omitted Proofs of \Cref{sec:mean}}\label{apx:mean}

\begin{proof}[\bf Proof of \Cref{thm:basic}]

\begin{align*}
|\widehat{\mu} - \mu| 
&\leq \sum^k_{\ell=0} a_\ell \cdot |\widehat{\mu}_\ell - \mu_\ell| + \E[X\cdot \1_{X > B}] \\
&\leq \sum^k_{\ell=0} \frac{\sqrt{a_\ell \cdot \mu_\ell} \log(1/ \delta)}{t} + \sum^k_{\ell =0}\frac{a_\ell \log(1/ \delta)^2}{t^2} + \E[X \cdot \1_{X>B}] \\
&\leq \frac{B \log(1/ \delta)}{t n} + \sum^k_{\ell = 1} \frac{\sqrt{2 \E[X^2 \cdot \1_{a_{\ell-1<X \leq a_\ell}}]} \log(1/ \delta)}{t} + \frac{2 B \log(1/ \delta)^2}{t^2} + \E[X \cdot \1_{X>B}] \\
&= \frac{B \log(1/ \delta)}{c n^2 \sqrt{\log(1/\delta)}} + \sum^k_{\ell = 1} \frac{\sqrt{2 \E[X^2 \cdot \1_{a_{\ell-1<X \leq a_\ell}}]} \log(1/ \delta)}{c n \sqrt{\log(1/\delta)}} + \frac{2 B \log(1/ \delta)^2}{c^2 n^2 \log(1/\delta)} + \E[X \cdot \1_{X>B}]  \\ 
&\leq \frac{\sqrt{2k} \cdot \sqrt{\sum^k_{\ell = 1} \E[X^2 \cdot \1_{a_{\ell-1} < X \leq a_\ell}]} \cdot \log(1/ \delta)}{c n \sqrt{\log(1/\delta)}} + \frac{3B \log(1/ \delta)^2}{c n^2 \sqrt{\log(1/\delta)}} + \E[X \cdot \1_{X>B}] \\
&= \frac{\sqrt{2}\cdot \sqrt{\E[X^2 \cdot \1_{|X| \leq B} ] } \cdot \log(1/ \delta)}{c n } + \frac{3 B \log(1/ \delta)^2}{c n^2 \sqrt{\log(1/\delta)}} + \E[X \cdot \1_{X>B}] \\
&\leq \frac{\sqrt{2} \cdot \sqrt{u B^{1-v}} \cdot \log(1/ \delta) }{c n} + \frac{3B \log(1/ \delta)^2}{c n^2 \sqrt{\log(1/\delta)}} + \frac{u}{B^v}   
\end{align*}
where the second inequality follows from Bernstein's inequality in Lemma \ref{alemma3}, the third inequality uses $a_0 \mu_0 \leq a_0^2 = \frac{B^2}{m^2}$ and $a_\ell \mu_\ell \leq 2 \E[X^2 \cdot \1_{a_{\ell -1} < X \leq a_\ell}]$ when $\ell \geq 1$, the fourth inequality uses the Cauchy-Schwarz inequality, the sixth inequality uses the fact that $\E[X^2 \cdot \1_{|X| \leq B}] \leq u B^{1-v}$ and $\E[X \cdot \1_{|X| > B}] \leq \frac{u}{B^v}$ when $\E[|X|^{1+v} \leq u$ and the equalities use the definition of $t$ and $k$.

\end{proof}

\begin{proof}[\bf Proof of \Cref{thm:mean_est}]
Using the fact that $X = Y_+  - Y_-$ and taking $B_t=\left(\frac{\sqrt{u}t}{\log (1/\delta)}\right)^{\frac{2}{1+v}}$, we obtain that
\begin{align*}
|\widehat{\mu}_{B_n}-\mu| &\leq |\widehat{\mu}_{Y_+} - \mu_{Y_+}| + |\widehat{\mu}_{Y_-} - \mu_{Y_-}| \\
&\leq \frac{\sqrt{2u B_n^{1-v}}\log(1/\delta)}{n}+\frac{3B_n \log^2(1/\delta)}{n^2}+ \frac{u}{B_n^v}\\
&\leq  C \frac{u^\frac{1}{1+v}(\log\frac{1}{\delta})^{\frac{2v}{1+v}}}{n^{\frac{2v}{1+v}}}
\end{align*}

\end{proof}

\section{Omitted Proofs of Section \ref{sec:MAB}}
\begin{proof}[\bf Proof of Theorem \ref{thm:reg1}]
    For each arm $i$, let $\mathcal{S}_i$ be the set of stages when arm $i$ is played, and denote $|\mathcal{S}_i|=M_i$. Initial
stages are not included in $\mathcal{S}_i$. According to Algorithm \ref{Algo:Heavy-QMAB}, each time we find arm $i$ by adopting UCB in Line 7 in some stages, $N_i$ is doubled subsequently. Then we play arm $i$
for consecutive $N_i$ rounds. This means that the number of rounds of each stage in $\mathcal{S}_i$ are $2^1, 2^2,\dots,2^{M_i}$.  In total, arm $i$ has been played for $2^{M_i+1}-1$ rounds. Because the total number of rounds is at most $T$, we have
$$\sum_{i=1}^K \left(2^{M_i+1}-1\right) \le T.$$
Because $2^x$ is a convex function in $x \in [0,+\infty)$, by Jensen’s inequality we have$$
\sum_{i=1}^K 2^{M_i+1} \geq K \cdot 2^{\frac{1}{K} \sum_{i=1}^K\left(M_i+1\right)}
$$

Then we have,
$$\sum_{i=1}^K M_i \le K\log\left(\frac{T+K}{K}\right)-K.$$
Since QTME is called for $K+\sum_{i=1}^K {M_i}$ times, by the union bound, with probability at least $1-K\delta \log \frac{T+K}{K}$ the output estimate of every invocation of QTME satisfies \eqref{eq:MeanEst}.We refer to the
 event as the good event and let $\mathcal{E}$ denote the good event. 

 Below we assume the good event holds. Recall that $i^*$ is the optimal arm and $i_s$ is the arm chosen by the algorithm during stage $s$. By the Line 7 of Algorithm \ref{Algo:Heavy-QMAB}, we have $$\hat{\mu}\left(i_s\right)+\beta_{i_s} \geq \hat{\mu}\left(i^*\right)+\beta_{i^*}.$$ Under good event, $${\mu}\left(i_s\right)+\beta_{i_s} \geq \hat{\mu}\left(i_s\right) \quad \text{and} \quad \hat{\mu}\left(i^*\right)+\beta_{i^*} \ge {\mu}\left(i^*\right).$$ Therefore, we obtain$$
\mu\left(i_s\right)+2 \beta_{i_s} \geq \hat{\mu}\left(i_s\right)+\beta_{i_s} \geq \hat{\mu}\left(i^*\right)+\beta_{i^*} \geq \mu\left(i^*\right),
$$and it follows that 
 \begin{equation}
\Delta_{i_s}:=\mu\left(i^*\right)-\mu\left(i_s\right) \leq 2 \beta_{i_s}.
\end{equation}
For each arm $i$, we denote by $\mathcal{R}(T;i)$ the contribution of arm $i$ to the cumulative regret over
$T$ rounds. By our notation above, arm $i$ is pulled in $M_i$ stages and the initialization stage. initialization stages it is pulled for $1$ time. In each stage of $\mathcal{S}_i$ it is pulled for $N_i= 2^1,\dots,2^{M_i}$ times respectively, and the reward gap $\Delta_i \le 2\beta_i$ in the last stage is $2C\frac{u^\frac{1}{1+v}(\log\frac{1}{\delta})^{\frac{2v}{1+v}}}{2^{\frac{2v}{1+v}M_i}}$. Note that the index of the stage in $\mathcal{S}_i$ does not influence the gap $\Delta_{i_s}$. Therefore, 
we can use $2C\frac{u^\frac{1}{1+v}(\log\frac{1}{\delta})^{\frac{2v}{1+v}}}{2^{\frac{2v}{1+v}M_i}}$ to bound the gap of $\mu(i^*)$ and $\mu(i_s)$. Thus, we have,$$
\begin{aligned}
    \mathcal{R}(T;i) &\le  \sum_{m=0}^{M_i} 2^m \cdot 2C\frac{u^\frac{1}{1+v}(\log\frac{1}{\delta})^{\frac{2v}{1+v}}}{2^{\frac{2v}{1+v}M_i}} \\
    & \le 4Cu^\frac{1}{1+v}(\log\frac{1}{\delta})^{\frac{2v}{1+v}} 2^{\frac{1-v}{1+v}M_i}\\
    & \le 4Cu^\frac{1}{1+v}(\log\frac{1}{\delta})^{\frac{2v}{1+v}} T^{\frac{1-v}{1+v}}
\end{aligned}$$
where the last inequality comes from the fact $M_i \le \log T$.
The cumulative regret is the summation of $ \mathcal{R}(T;i)$ for $i\neq i^*$. We have $$\mathcal{R}_T =\sum_{i\neq i^*}\mathcal{R}(T;i)  \le 4Cu^\frac{1}{1+v}(\log\frac{1}{\delta})^{\frac{2v}{1+v}} (K-1)T^{\frac{1-v}{1+v}}.$$

Since we have good event $\mathcal{E}$ with probability $1-K\delta \log \frac{T+K}{K}$, we can obtain the regret bound by taking $\delta=\frac{1}{T}$
$$\begin{aligned}
    &\mathbb{E}[\mathcal{R}_T]\\
    \le &\left(1-\frac{K}{T} \log \frac{T+K}{K}\right)\mathbb{E}[[\mathcal{R}_T |\mathcal{E}]+\frac{K}{T} \log \frac{T+K}{K}\mathbb{E}[[\mathcal{R}_T |\bar{\mathcal{E}}]\\
    \le & 4Cu^\frac{1}{1+v}(\log T)^{\frac{2v}{1+v}} (K-1)T^{\frac{1-v}{1+v}}+\\
    &T \cdot Cu^\frac{1}{1+v}(\log T)^{\frac{2v}{1+v}}\frac{K}{T} \log \frac{T+K}{K}\\
    \le& O\left(u^\frac{1}{1+v}  K T^{\frac{1-v}{1+v}} \log T \right)
\end{aligned}$$
\end{proof}
\section{Omitted Proofs of Section \ref{sec:SLB}}
\begin{proof}[\bf Proof of Lemma \ref{lem:stage}]
    We show that if Algorithm \ref{Algo:Heavy-QLinMAB} executes $m$ stages, then at least $T$ rounds are played. We first give a lower bound for $\sum_{k=1}^m \frac{1}{\epsilon_k^2}$. For $k \ge 0$,
    $$
\begin{aligned}
&\operatorname{det}\left(V_{k+1}\right) \\
& =\operatorname{det}\left(V_k+\frac{1}{\epsilon_{k+1}^2} a_{k+1} a_{k+1}^{\top}\right) \\
& =\operatorname{det}\left(V_k^{1 / 2}\left(I+\frac{1}{\epsilon_{k+1}^2} V_k^{-1 / 2} a_{k+1} a_{k+1}^{\top} V_k^{-1 / 2}\right) V_k^{1 / 2}\right) \\
& =\operatorname{det}\left(V_k\right) \operatorname{det}\left(I+\frac{1}{\epsilon_{k+1}^2} V_k^{-1 / 2} a_{k+1} a_{k+1}^{\top} V_k^{-1 / 2}\right)\\
& =\operatorname{det}\left(V_k\right)\left(1+\left\|\frac{1}{\epsilon_{k+1}} V_k^{-1 / 2} a_{k+1}\right\|^2\right) \\
& =\operatorname{det}\left(V_k\right)\left(1+\frac{1}{\epsilon_{k+1}^2}\left\|a_{k+1}\right\|_{V_k^{-1}}^2\right) \\
& =2 \operatorname{det}\left(V_k\right).
\end{aligned}
$$
Thus,$\operatorname{det}\left(V_m\right)=2^m \operatorname{det}\left(V_0\right)=2^m \lambda^d$. On the other hand,$$
\operatorname{tr}\left(V_m\right)=d \lambda+\sum_{k=1}^m \frac{\left\|a_k\right\|^2}{\epsilon_k^2} \leq d \lambda+\sum_{k=1}^m \frac{L^2}{\epsilon_k^2}.
$$

By the trace-determinant inequality,
$$
d \lambda+\sum_{k=1}^m \frac{L^2}{\epsilon_k^2} \geq \operatorname{tr}\left(V_m\right) \geq d \cdot \operatorname{det}\left(V_m\right)^{1 / d}=d \lambda \cdot 2^{m / d} .
$$
Hence,
$$
\sum_{k=1}^m \frac{1}{\epsilon_k^2} \geq \frac{d \lambda}{L^2}\left(2^{m / d}-1\right).
$$
Since the $k$-th stage contains $\frac{C u^{\frac{1}{2v}}\log(m/\delta)}{\epsilon_k^{\frac{1+v}{2v}}}$ rounds, the first $m$ stages contain $\sum_{k=1}^m \frac{C u^{\frac{1}{2v}}\log(m/\delta)}{\epsilon_k^{\frac{1+v}{2v}}}$  rounds in total. By the above argument, we have for $v \in [\frac{1}{3},1]$
$$\begin{aligned}
    \sum_{k=1}^m \frac{C u^{\frac{1}{2v}}\log(m/\delta)}{\epsilon_k^{\frac{1+v}{2v}}}&\ge \sum_{k=1}^m \frac{1}{\epsilon_k^{\frac{1+v}{2v}}} \\
    &\ge \left(\sum_{k=1}^m \frac{1}{\epsilon_k^2}\right)^{\frac{1+v}{4v}}\\
    & \ge \left(\frac{d \lambda}{L^2}\left(2^{m / d}-1\right)\right)^{\frac{1+v}{4v}}\\
    & \ge T.
\end{aligned}$$
 where the second inequality follows from Lemma \ref{lem:ineq1} for $v \in [\frac{1}{3},1]$.
\end{proof}

\begin{proof}[\bf Proof of Theorem \ref{thm:LinearReg}]
    In stage $s$, the algorithm plays action $a_s$ for $N_s=\frac{C u^{\frac{1}{2v}}\log(m/\delta)}{\epsilon_s^{\frac{1+v}{2v}}}$ rounds. 
    The regret in each round is $\left(a^*-a_s\right)^{\top} \theta^*$. By the choice of $\left(a_s, \tilde{\theta}_s\right)$, $$
\left(a^*\right)^{\top} \theta^* \leq a_s^{\top} \tilde{\theta}_s .$$
Therefore, by Cauchy-Schwarz inequality,$$
\begin{aligned}
&\left(a^*-a_s\right)^{\top} \theta^* \\
&\leq a_s^{\top}\left(\tilde{\theta}_s-\theta^*\right) \leq\left\|a_s\right\|_{V_{s-1}^{-1}}\left\|\tilde{\theta}_s-\theta^*\right\|_{V_{s-1}}\\
& \leq\left\|a_s\right\|_{V_{s-1}^{-1}}\left(\left\|\tilde{\theta}_s-\hat{\theta}_{s-1}\right\|_{V_{s-1}}+\left\|\hat{\theta}_{s-1}-\theta^*\right\|_{V_{s-1}}\right) \\
& =\epsilon_s \cdot\left(\left\|\tilde{\theta}_s-\hat{\theta}_{s-1}\right\|_{V_{s-1}}+\left\|\hat{\theta}_{s-1}-\theta^*\right\|_{V_{s-1}}\right) .
\end{aligned}$$
By the choice of $\tilde{\theta}_s$ and Lemma \ref{lem:confiSet}, with probability at least $1-\delta$, for all stages $s \ge 1$, both $\tilde{\theta}_s$ and ${\theta}^*$ lie in $\mathcal{C}_{s-1}$. Thus,$$
\left(a^*-a_s\right)^{\top} \theta^* \leq 2 \epsilon_s \cdot\left(\lambda^{1 / 2} S+\sqrt{d(s-1)}\right).$$
The cumulative regret in stage s is therefore bounded by 
$$\begin{aligned}
    &\frac{2Cu^{\frac{1}{2v}}\left(\lambda^{1 / 2} S+\sqrt{d(s-1)}\right)\log(m/\delta)}{\epsilon_s^{\frac{1-v}{2v}} }\\
=&2 \left(\lambda^{1 / 2} S+\sqrt{d(s-1)}\right)\left(C \log\frac{m}{\delta}\right)^{\frac{2v}{1+v}}u^{\frac{1}{1+v}}N_s^{\frac{1-v}{1+v}}.
\end{aligned}$$
which follows from the relationship of $N_s=\frac{C u^{\frac{1}{2v}}\log(m/\delta)}{\epsilon_s^{\frac{1+v}{2v}}}$.

Since there are at most $m$ stages by Lemma \ref{lem:stage}, the cumulative regret over all stages and rounds satisfies$$
\begin{aligned}
    &\mathcal{R}_T \\
    & \le 2 \left(\lambda^{1 / 2} S+\sqrt{d(m-1)}\right)\left(C \log\frac{m}{\delta}\right)^{\frac{2v}{1+v}}u^{\frac{1}{1+v}}\sum_{k=1}^m N_s^{\frac{1-v}{1+v}}\\
    & \le  2 \left(\lambda^{1 / 2} S+\sqrt{d(m-1)}\right)\left(C \log\frac{m}{\delta}\right)^{\frac{2v}{1+v}}u^{\frac{1}{1+v}}m T^{\frac{1-v}{1+v}}\\
    & \le O\left( \left(\lambda^{1 / 2} S m+{d^{\frac{1}{2}}m^{\frac{3}{2}}}\right)\left( \log\frac{m}{\delta}\right)^{\frac{2v}{1+v}}u^{\frac{1}{1+v}} T^{\frac{1-v}{1+v}}\right)\\
    & \le O\left( d^2  u^{\frac{1}{1+v}} T^{\frac{1-v}{1+v}}\log^{3/2} \left(\frac{L^2 T^{\frac{4v}{1+v}}}{d \lambda}+1\right) \right.\\
    & \quad \cdot \left.\left( \log\frac{d \log \left(\frac{L^2 T^{\frac{4v}{1+v}}}{d \lambda}+1\right)}{\delta}\right)^{\frac{2v}{1+v}} 
    \right)
\end{aligned}$$

For expected regret bound, let $\mathcal{E}$ be the event that the above bound holds. Note that for any $a_1,a_2 \in \mathcal{A}$, $$
\begin{aligned}
    \left|\left(a_1-a_2\right)^{\top} \theta^*\right| &\leq\left\|a_1-a_2\right\|_2\left\|\theta^*\right\|_2 \\&\leq\left(\left\|a_1\right\|_2+\left\|a_2\right\|\right)\left\|\theta^*\right\|_2 \\
    &\leq 2 L S .
\end{aligned}$$
Then, by taking $\delta=\frac{1}{T}$ we have $$
\begin{aligned}
   &\mathbb{E}[\mathcal{R}_T]\\
   =&\mathbb{E}[\mathcal{R}_T \mid \mathcal{E}] \operatorname{Pr}[\mathcal{E}]+\mathbb{E}[R(T) \mid \overline{\mathcal{E}}] \operatorname{Pr}[\overline{\mathcal{E}}] \\
   \le &  O\left( \left(\lambda^{1 / 2} S m+{d^{\frac{1}{2}}m^{\frac{3}{2}}}\right)\left( \log\frac{m}{\delta}\right)^{\frac{2v}{1+v}}u^{\frac{1}{1+v}} T^{\frac{1-v}{1+v}}\right)+2LST \frac{m}{T}\\
   = & O\left( d^2  u^{\frac{1}{1+v}} T^{\frac{1-v}{1+v}}\log^{3/2} \left(\frac{L^2 T^{\frac{4v}{1+v}}}{d \lambda}+1\right) \cdot \left( \log\frac{d \log \left(\frac{L^2 T^{\frac{4v}{1+v}}}{d \lambda}+1\right)}{\delta}\right)^{\frac{2v}{1+v}} \right)
\end{aligned}$$
where we omit $\log \log(\cdot)$ term.
\end{proof}

\section{More Numerical Experiments}\label{apx:experiment}
\begin{figure}[!ht]
    \centering
    \subfigure{
    \includegraphics[width=0.32\textwidth]{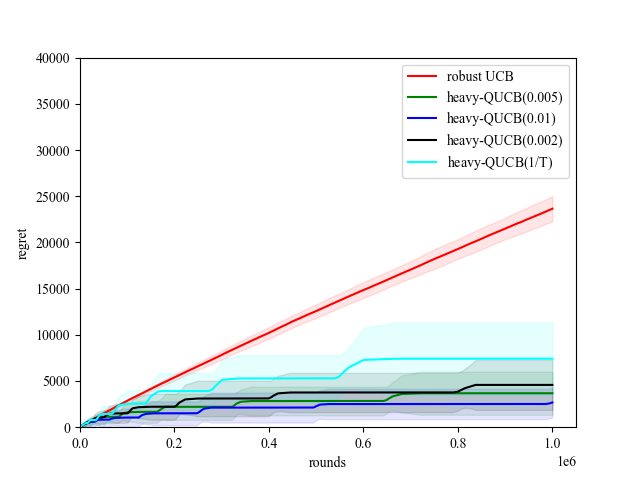}}
    \subfigure{\includegraphics[width=0.32\textwidth]{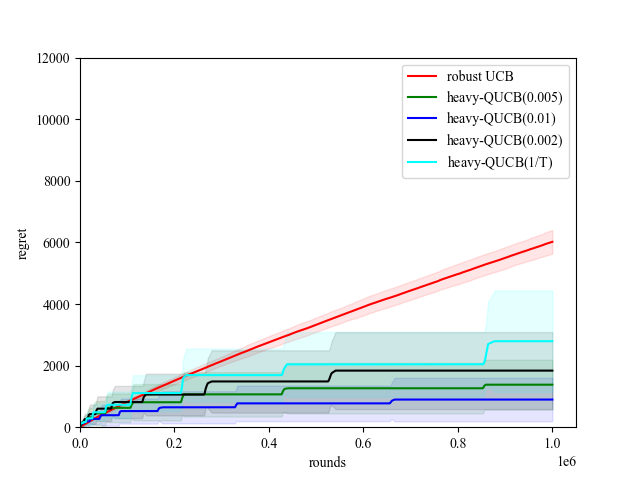}}
    \subfigure{\includegraphics[width=0.32\textwidth]{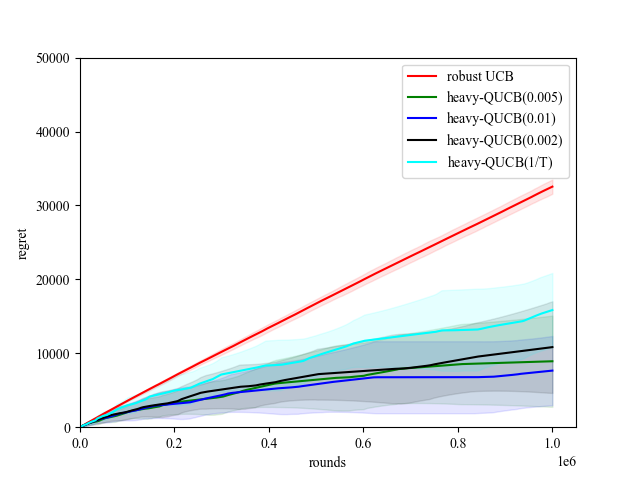}}
    \caption{Comparison between robust UCB and Heavy-QUCB with $v=0.2$}
    \label{fig:qmab_0.2}
\end{figure}

\begin{figure}[!ht]
    \centering
    \subfigure{\includegraphics[width=0.32\textwidth]{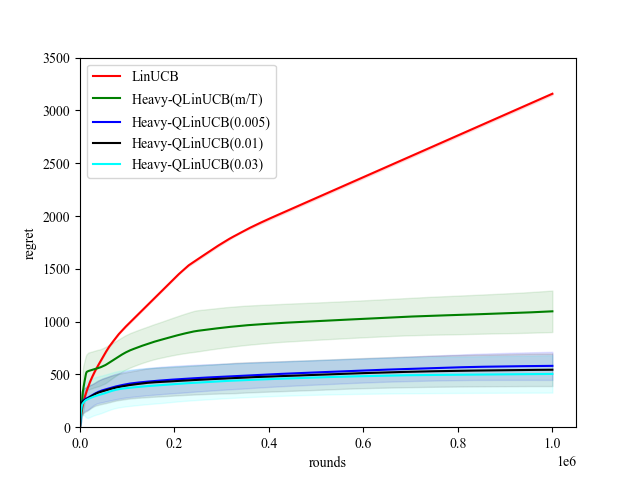}}
    \subfigure{\includegraphics[width=0.32\textwidth]{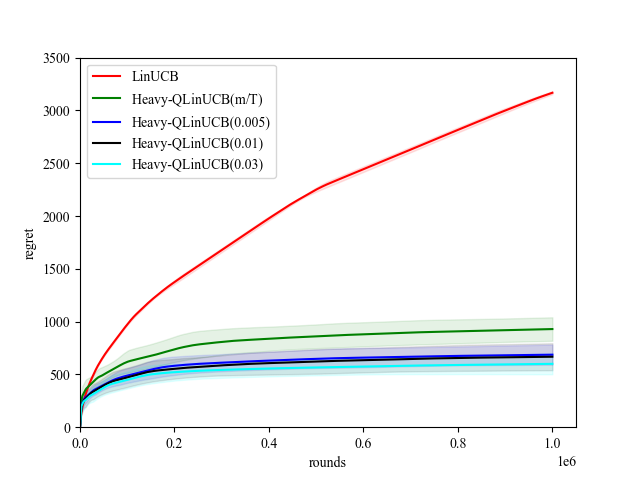}}
    \subfigure{\includegraphics[width=0.32\textwidth]{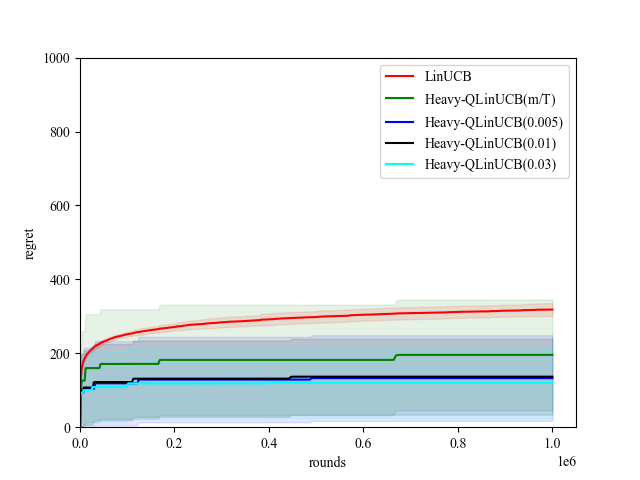}}
    \caption{Simulation for the QSLB setting with different $\delta$}
    \label{fig:qslbmy_label}
\end{figure}
\noindent \textbf{Simulation of quantum mean estimator.} As mentioned in \cite{wan2022quantum}, the closed form of the measurement output distribution given in Theorem 11 of \cite{brassard2002quantum} enables a classical way to the simulation of quantum amplitude estimation algorithm. We use this strategy to simulate our quantum mean estimator. 

\noindent \textbf{QSLB setting.} For this setting, we set $T = 10^6$ and study 3 instances: (1) $\theta^*_1 = (\cos(0.35\pi),\sin(0.35\pi))$ as in \cite{wan2022quantum}; (2) $\theta^*_2 = (\cos(\pi/6),\sin(\pi/6))$; (3) $\theta^*_3 = (\cos(5\pi/6),\sin(5\pi/6))$. As for the action set, we use a finite one that consists of 50 actions equally spaced on the positive part of the unit circle. 

As for the solver for $(a_s,\tilde{\theta}_s )$ in \Cref{Algo:Heavy-QLinMAB}, there is no efficient solver in general. However, when considering the finite action set, a straightforward approach to solving this is to enumerate all possible solutions in a small finite set. The existence of a small finite solution set with high probability is shown by \Cref{lem:confiSet} and hence, this enables the feasibility of using enumeration on solving for $(a_s, \tilde{\theta}_s)$. Following a similar idea in \cite{wan2022quantum}, the upper bound on $\|\theta - \hat{\theta}_{s-1}\|_{V_{s-1}}$ can be calculated. For completeness, $\|\theta - \hat{\theta}_{s-1}\|_{V_{s-1}} \leq \lambda^{1/2}S + \sqrt{(s-1) \|W^{1/2}_{s-1} A_{s-1} V^{-1}_{s-1} A^\top_{s-1} W^{1/2}_{s-1}\|_2}$.

We compare our algorithm with the well-known LinUCB \cite{lattimore2020bandit} for the case $v = 1$ and $\lambda = 1$, which is shown in \Cref{fig:qslbmy_label}. From the results, we can see that Heavy-QLinUCB has lower regret than LinUCB. By comparison with QLinUCB in \cite{wan2022quantum} on the first instance $\theta^*_1$, we can also observe that Heavy-QlinUCB has slightly better performance which matches the comparison in \Cref{tab:1}.
\end{document}